\documentclass[conference]{IEEEtran}
\IEEEoverridecommandlockouts
\usepackage{cite}
\usepackage{amsmath,amssymb,amsfonts}
\usepackage{algorithmic}
\usepackage{graphicx}
\usepackage{textcomp}
\def\BibTeX{{\rm B\kern-.05em{\sc i\kern-.025em b}\kern-.08em
    T\kern-.1667em\lower.7ex\hbox{E}\kern-.125emX}}
    \usepackage{amsthm}
    \PassOptionsToPackage{usenames,dvipsnames,svgnames}{xcolor} 
\newtheorem{prop}{Proposition}
\begin{document}
\title{Robotic Mobility Diversity Algorithm with Continuous Search Space\\
}
\title{Robotic Mobility Diversity Algorithm with Continuous Search Space\\
}

\author{\IEEEauthorblockN{Daniel Bonilla Licea$^{*}$, Des McLernon{$^{\ddagger}$}, Mounir Ghogho$^{{\ddagger},\ast}$, Edmond Nurellari{$^{\dagger}$}, and Syed Ali Raza Zaidi$^{\ddagger}$}
\IEEEauthorblockA{$*${International University of Rabat, TICLab, Morocco} \\
{$\dagger$ School of Engineering, University of Lincoln, UK }\\
{$\ddagger$ School of Electronic and Electrical Engineering, University of Leeds, UK }\\
daniel.bonilla-licea@uir.ac.ma, d.c.McLernon@leeds.ac.uk, m.ghogho@ieee.org, enurellari@lincoln.ac.uk, elsarz@leeds.ac.uk}
}

\maketitle

\begin{abstract}
Small scale fading makes the wireless channel gain vary significantly over small distances and in the context of classical communication systems it can be detrimental to performance. But in the context of mobile robot (MR) wireless communications, we can take advantage of the fading using a mobility diversity algorithm (MDA) to deliberately locate the MR at a point where the channel gain is high. There are two classes of MDAs. In the first class, the MR explores various points, stops at each one to collect channel measurements and then locates the best position to establish communications. In the second class the MR moves, without stopping, along a continuous path while collecting channel measurements and then stops at the end of the path. It determines the best point to establish communications. Until now, the shape of the continuous path for such MDAs has been arbitrarily selected and currently there is no method to optimize it. In this paper, we propose a method to optimize such a path. Simulation results show that such optimized paths provide the MDAs with an increased performance, enabling them to experience higher channel gains while using less mechanical energy for the MR motion.
\end{abstract}


\section{Introduction}
\label{sec:intro}
Small-scale fading can degrade wireless communications links and therefore diversity techniques must be used for compensation. Now, a new class of diversity, referred to as mobility diversity, has recently started to be developed in the context of robotic communications. Algorithms implementing this type of diversity are called mobility diversity algorithms (MDAs) \cite{r10}, \cite{r11}.

MDAs operate by making the MR measure the channel gain over an exploration space in its close vicinity and then, after having gathered channel gain measurements, it determines the optimum position from which to establish communications. The exploration space can be either continuous or discrete.

Discrete exploration spaces consists of a finite number of stopping points. In this case, the MR moves from one point to the next one while stopping at each point. We will refer to such MDAs as stopping points based MDAs. Different configurations for such points have been proposed arbitrarily (e.g., \cite{r11}, \cite{r12}) but without any concrete theory behind their choice. So, in \cite{r9} we have considered such a problem and  presented a method to design in a systematic way the spatial configuration for any number of stopping points.

Now, continuous exploration spaces consist of a continuous path. Here the MR moves along a continuous path without stoppping until its end while collecting wireless channel samples. We will refer to MDAs using this type of path as ``continuous MDAs'' (CMDAs). In \cite{r0} the authors implemented experimentally an MDA with a circular exploration path. In \cite{r13} the authors implemented experimentally a continuous path MDA where the MR samples only at positions along the path that produce statistically independent channels measurements. The authors explored a circular path and also a path that produces samples arranged into a hexagonal lattice contained inside a circle. In \cite{r6} the authors implemented experimentally a continuous MDA with linear, circular, spiral and random paths. Then in \cite{r8} we proposed a continuous MDA with a linear path where the length was optimized. 

As can be seen, there is not still a general method to design the exploration space for CMDAs. The interest behind the development of such a method is that it can improve the performance of CMDAs which can be more energy efficient than the stopping point based MDAs. This is because when executing a CMDA, the MR moves continuously without stopping along the exploration path and thus taking advantage of the momentum and so reducing the energy spent in motion.

In this paper, we consider the problem of a MR equipped with a single antenna required to establish communications, in a static environment, with a stationary node (also equipped with a single antenna) through a wireless channel experiencing small-scale fading. To take advantage of the fading, we design a CMDA. As stated above, there does not yet exist a theory that optimises the actual shape of the exploration path concerning CMDAs. So, in this paper, for the first time we propose a solution to such a problem.  

Section \ref{sec:Model} presents the model for the wireless channel. Then in section \ref{sec:CSSMDA}, we develop the CMDA and in section \ref{sec:Simulations} we show the performance of the CMDA under different conditions. Finally in section \ref{sec:Conclusions}, we present our conclusions.

\section{System Model}
\label{sec:Model}
\subsection{Wireless Channel Model}
\label{sec:Model:channel}
We assume that there is no line of sight between the stationary node\footnote{The stationary node can be a base station or another MR which remains stationary during the CMDA execution.} and the MR; that the signal transmitted by the stationary node to the MR is narrowband (i.e., the bandwidth of the signal is significantly smaller than the radio frequency carrier frequency used in the transmission); that the MR's environment is stationary (i.e., it does not change with time during the execution of the MDA) and presents a large number of scatterers (which produce the small-scale fading\footnote{Note that small-scale fading is usually modelled as a random process in the communications literature \cite{M05}.}). Consequently, the wireless channel is time invariant (at a given position) and exhibits Rayleigh\footnote{Rayleigh fading means that the amplitude of the fading in the wireless channel is Rayleigh distributed.} flat (i.e., frequency independent) fading \cite{M05}. Thus, the signal received by the MR at time instant $t$, when located at $\mathbf{p}(t)$, is:
\begin{equation}
\label{eq:1.2.1}
z(t)=h(\mathbf{p}(t))\times w(t)+n(t),
\end{equation}
where $w(t)$ is the narrowband signal transmitted by the stationary node, $n(t)\sim\mathcal{CN}(0,\sigma_n^2)$ is the additive white Gaussian noise generated at the MR's receiver and $h(\mathbf{p}(t))$ represents the small-scale fading. We will assume Jakes' model \cite{J94} for the small scale fading term. This is a probabilistic model which consists of $h(\mathbf{p}(t))\sim \mathcal{CN}\left(0,1\right)$ and the spatial normalized correlation function given by:
\begin{equation}
\label{eq:1.2.2}
r(\mathbf{p},\mathbf{q}) =  \mathbb{E}\left[h(\mathbf{p})h^*(\mathbf{q})\right]=J_0\left({2\pi\|\mathbf{p}-\mathbf{q}\|_2}/{\lambda}\right),
\end{equation}
where $J_0(\cdot)$ is the Bessel function of zeroth order and first kind,  $\lambda$ is the wavelength used in the RF transmission and $\mathbf{p},\mathbf{q}\in\mathbb{R}^2$  are any two points in the 2-D space.

\section{Continuous Path Mobility Diversity Algorithm}
\label{sec:CSSMDA}
The CMDA is a class of MDA in which the MR explores without stopping a continuous path $\mathcal{P}$ of desired length $L_p$. To do so, the stationary node which communicates with the MR initially transmits a pure tone\footnote{Thus the (lowpass, complex equivalent) signal transmitted by the stationary node $w(t)=K$ (in (\ref{eq:1.2.1})) where $K$ is the amplitude of the received tone. } that allows the MR to estimate the channel gain along the path being explored. We will refer to this period of time as the exploration phase. During this phase the MR follows a continuous path $\mathcal{P}$ while estimating the wireless channel at sampling points\footnote{We define $\mathcal{S}\subseteq\mathcal{P}$ as the set of all sampling points.} uniformly distributed all along the path. Then, once the MR reaches the end of the path it determines the sampling point with highest channel gain $\mathbf{q}_{opt}\in\mathcal{S}$. Finally, the MR moves to $\mathbf{q}_{opt}$ to establish communications with the stationary node. This second part will be referred to as the positioning phase. Once the exploration phase finishes the stationary node acts as a receiver and waits for a reply from the MR to establish communication. In the next subsection we will show how to optimize the shape of the continuous exploration path $\mathcal{P}$ for the CMDA.
\subsection{Optimum Exploration Path}
\label{sec:CSSMDA:path}
In this section, we optimize the continuous path $\mathcal{P}$ for a given desired length $L_p$. Since the objective of the CMDA is to enable the MR to obtain a high channel gain at $\mathbf{q}_{opt}$, then $\mathcal{P}$ must be optimised to maximise $\mathbb{E}[|h(\mathbf{q}_{opt})|]$ which depends on the set of sampling points $\mathcal{S}$ which in turn depends on the shape and the length of the continuous path $\mathcal{P}$ as well as on the spatial sampling rate used during the exploration phase.

Although an analytical expression for $\mathbb{E}[|h(\mathbf{q}_{opt})|]$ cannot be obtained as a function of the continuous path $\mathcal{P}$, we can still optimize the path after we discretize it. To do this, we define the set of path points $\mathcal{D}_N=\left\{\mathbf{d}_j\right\}_{j=1}^N$ (not to be confused with the set of sampling points $\mathcal{S}$) associated with the following constraint:
\begin{equation}
\label{eq:2.1.1}
\|\mathbf{d}_j-\mathbf{d}_{j+1}\|_2={L_p}/(N-1),\ \ j=1,2,\cdots,N-1.
\end{equation}
Note that, for any given continuous path $\mathcal{P}$ with length $L_p$ and a sufficiently large value $N$, there exists a set $\mathcal{D}_N$ with the constraint (\ref{eq:2.1.1}) such that $\mathcal{I}\left\{\mathcal{D}_N\right\} \approx \mathcal{P}$. Here $\mathcal{I}\left\{\mathcal{D}_N\right\}$ is an interpolation of the set of points $\mathcal{D}_N$. Therefore, the set of points $\mathcal{D}_N$ can be thought of as the discretized version of the continuous path $\mathcal{P}$ which is obtained by  interpolating of the points of that set. So, to optimize the continuous path $\mathcal{P}$ we optimize the set of path points\footnote{Note that we do not use the sampling points $\mathcal{S}$ for this optimization in order to decouple this process from the spatial sampling rate.} $\mathcal{D}_N$.

The set $\mathcal{D}_N$ is optimized to maximize $\mathbb{E}[|h(\mathbf{q}_{opt})|]$ with $\mathbf{q}_{opt}\in\mathcal{D}_N$. Mathematically this is similar to the optimization of the stopping points configuration which was solved in \cite{r9} by minimizing the amount of correlation among the wireless channels at the different stopping points. So, we use the same approach here to optimize the path points:
\begin{equation}
\label{eq:2.1.2}
    \begin{array}{l}
\displaystyle\min_{\psi_1,\psi_2,\cdots,\psi_{N-1}} \sum_{m=1}^N\sum_{n=1}^NJ_0^2\left(\frac{2\pi\|\mathbf{d}_m-\mathbf{d}_n\|_2}{\lambda}\right)\\
{\rm s.t.}\\
\|\mathbf{d}_{j+1}-\mathbf{d}_{j}\|_2=\frac{L_p}{N-1},\ \ j=1,2,\cdots,N-1,\\
\measuredangle\{\mathbf{d}_{j+1}-\mathbf{d}_{j}\}=\psi_j\in[0,2\pi),\ \ j=1,2,\cdots,N-1.
\end{array}
\end{equation}
The first constraint in (\ref{eq:2.1.2}) comes from (\ref{eq:2.1.1}) and the angles $\psi_j$ are the only independent variables\footnote{Since only the shape of the path matters, the angle $\psi_1$ can take any value without eliminating any solution in (\ref{eq:2.1.2}).} which define the shape of $\mathcal{D}_N$. Now, since the optimum continuous path $\mathcal{P}$ will be derived by interpolation from the $\mathcal{D}_N$ which is optimized to minimize the correlation among the channels at $\mathcal{D}_N$, then we will refer to $\mathcal{P}$ as the minimum correlation path (MCP).
\begin{prop}
\label{prop1}
For $L_p/\lambda> z_0$, where $z_0$ is the smallest value of $z$ that satisfies $J_0^2(2\pi z)=0$, the MCP is the straight line path.
\end{prop}
\begin{proof}
The function $J_0^2(2\pi x)$ is decreasing for $x\leq z_0$. Then, if $L_p/\lambda\leq z_0$ the norm $\|\mathbf{C}_N\|_F^2$ is minimized when $\|\mathbf{d}_m-\mathbf{d}_n\|_2$ is maximized for all $n,m$. This occurs when $\psi_j=\psi_1$ for all $j=1,2\cdots N-1$. Hence, when $L_p/\lambda\leq z_0$, (\ref{eq:2.1.2}) is solved when all the path points $\mathcal{D}_N$ lie on a straight line.\\
\end{proof}

When $L_p/\lambda> z_0$ the cost function in (\ref{eq:2.1.2}) becomes non convex and presents multiple minima. But since the optimization variables are bounded angles we can solve (\ref{eq:2.1.2}) numerically using simulated annealing \cite{b6}.

The interpolation of the path points $\mathcal{D}_N$ can be implemented using splines \cite{b5} to obtain the continuous exploration path $\mathcal{P}$. We want the MR to be able to move through the continuous exploration path without stopping until it reaches the end of the path. This is because moving continuously without stopping along the path is more energy efficient than stopping various times along the path. To achieve this, the continuous path $\mathcal{P}$ should be significantly smooth. So we use second-order\footnote{Second-order splines are sufficient to allow the MR to traverse the continuous path $\mathcal{P}$ with continuous velocity and without needing to stop due to abrupt direction changes.} splines to perform the interpolation of the path points $\mathcal{D}_N$ to obtain $\mathcal{P}$. The parametrized function for the continuous path $\mathcal{P}$ is then obtained by:
\begin{equation}
\label{eq:2.1.3}
\mathbf{g}(s)=\bigg\{
\begin{array}{l}
\Pi_j(s-j),\ \forall s\in[j,j+1),\ j=1,2,\dots,N-1\\
\Pi_{N-1}(1),\ {\rm otherwise}
\end{array}
\end{equation}
where $s\in[0,N-1]$ is a free parameter and $\Pi_j(s)$ is a two-dimensional second-order polynomial vector satisfying the following restrictions:
\begin{eqnarray}
\label{eq:2.1.4a}
\Pi_j(1)=\Pi_{j+1}(0)=\mathbf{d}_{j+1},\ \ j=1,2,\cdots,N-2\\
\label{eq:2.1.4b}
\frac{{\rm d}\Pi_j(s)}{{\rm d}s}\bigg|_{s=1}=\frac{{\rm d}\Pi_{j+1}(s)}{{\rm d}s}\bigg|_{s=0},\ \ \ j=1,2,\cdots,N-2\\
\label{eq:2.1.4c}
\Pi_1(0)=\mathbf{d}_1,\ \ \ \Pi_{N-1}(1)=\mathbf{d}_N.
\end{eqnarray}
The restriction (\ref{eq:2.1.4a}) ensures the continuity of the exploration path; restriction (\ref{eq:2.1.4b}) ensures the smoothness of the exploration path $\mathcal{P}$; and (\ref{eq:2.1.4c}) ensures that the starting and ending points of the exploration path $\mathcal{P}$ are $\mathbf{d}_1$ and $\mathbf{d}_N$ respectively.

\begin{figure}[htp!]           
\centerline{{\includegraphics[clip, trim ={5mm 2mm 3mm 1mm},width=90mm ,height=80mm]{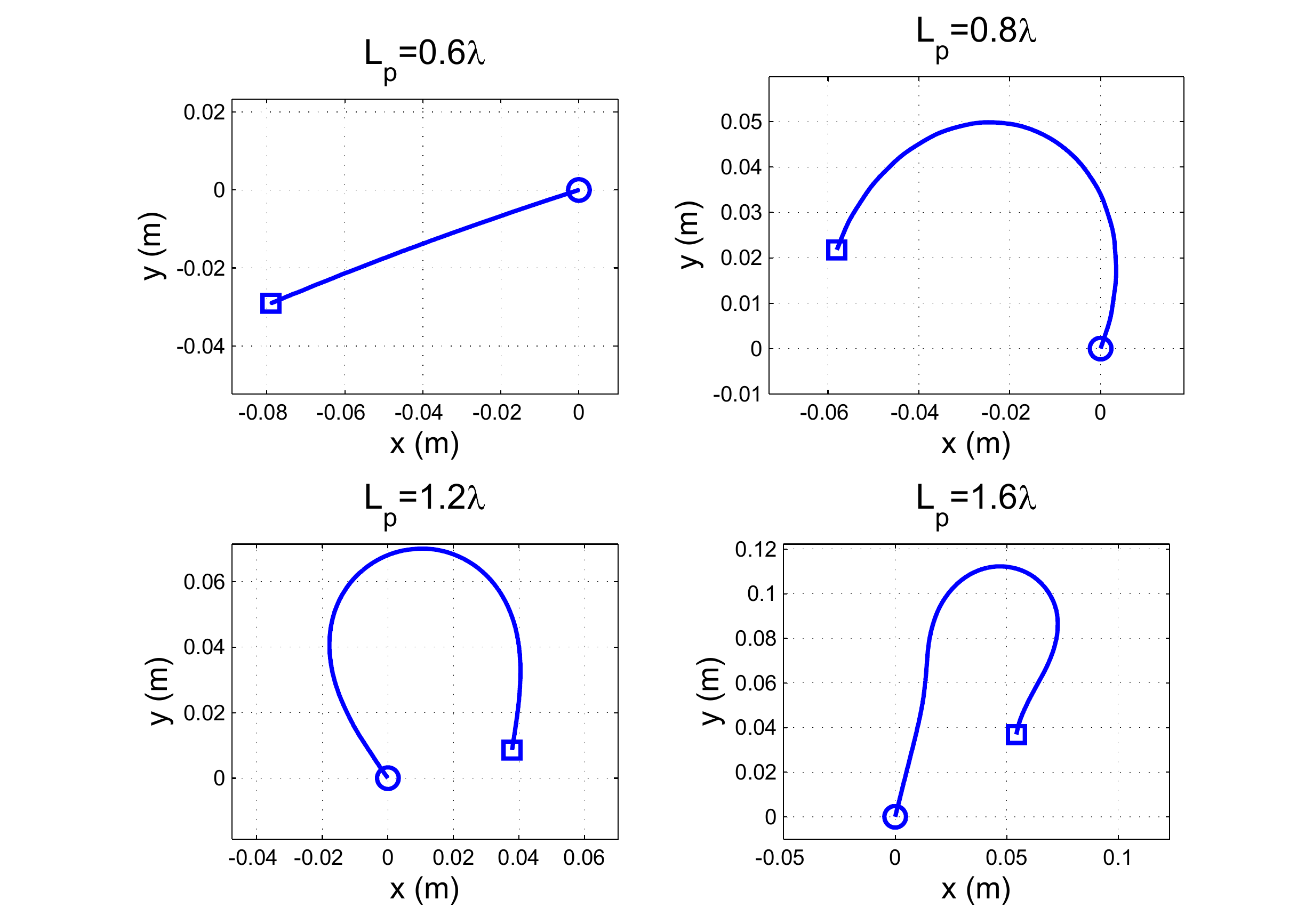}}}
    \caption[Global probability of detection for different P_t values]{\label{Figure1}
   Optimized continuous path $\mathcal{P}$ for $N=25$, $\lambda=14.02$ cm and different values of $L_p$. The start point $\mathbf{d}_1$ and the end point $\mathbf{d}_N$ are represented with a circle and a square respectively.}
\end{figure}

Now, $L_p$ (see (\ref{eq:2.1.1})) is the desired length while $L_p'$ is the actual length of the continuous exploration path $\mathcal{P}$ and is given by $L_p'=\int_{0}^{N-1}\left\|\frac{{\rm d}\mathbf{g}(s)}{{\rm d}s}\right\|_2{\rm d}s$. As a consequence, with exception of the linear path, $L_p'$ is slightly greater than $L_p$.

In Fig. \ref{Figure1} we observe how the shape of the MCP changes as $L_p$ increases. We would like to make it clear that as $L_p$ increases, solving (\ref{eq:2.1.2}) becomes increasingly difficult since the number of local minima increases. So, in this paper we have solved (\ref{eq:2.1.2}) only up to $L_p=1.8\lambda$.

Finally, regarding the exploration sense of the path $\mathcal{P}$ (i.e., which ends of the path constitute the starting and the ending points) it only affects the average distance traveled during the positioning phase (i.e., the average distance from the ending point to $\mathbf{q}_{opt}$). To reduce this distance, we choose as ending point the point from $\{\mathbf{d}_1, \mathbf{d}_N\}$ that minimizes its average distance to the path points $\mathcal{D}_N$.

\subsection{$\mathbf{q}_{opt}$ Determination and Channel Estimation}
\label{sec:CSSMDA:estimation}
The main advantage of the CMDA with respect to stopping points based MDAs is that it takes advantage of the MR inertia (by stopping only at the end of the path) to reduce the energy spent in motion. But as a consequence, it can only take a single wireless channel measurement per sampling point and thus it is more vulnerable to noise compared to the stopping points based MDAs. To compensate for this, when the exploration phase is over, the noisy signals received during the exploration phase must be smoothed to obtain good wireless channel estimates to determine $\mathbf{q}_{opt}$. Hence, the channel estimation is done using the following smoother:
\begin{equation}
\label{eq:2.3.3}
\hat{h}_s(\mathbf{p}(t_s(k)),\mathcal{S}_k(d))=\sum_{j\in\mathcal{S}_k(d)}a_{k,j}z(t_s(j))
\end{equation}
where $t_s(k)$ is the $k$th sampling instant\footnote{Sampling instants are chosen so that the spatial samples are uniformly distributed along the path $\mathcal{P}$ and so the temporal sampling rate is not uniform in general.}; $\mathcal{S}_k(d)$ is the set of all integers $j$ that satisfy $\|\mathbf{p}(t_s(k))-\mathbf{p}(t_s(j))\|\leq d$; and $\hat{h}_s(\mathbf{p}(t_s(k)),\mathcal{S}_k(d))$ is the estimate for $h(\mathbf{p}(t_s(k)))$ using the measurements collected during the sampling instants $\{t_s(j)\}_{j\in\mathcal{S}_k(d)}$.

By using the appropriate value of $d$ in $\mathcal{S}_k(d)$, the MR will only use channel measurements that are highly correlated with $h(\mathbf{p}(t_s(k)))$ and neglect the rest of those with low correlation, thus reducing the computational load of the estimation process.

Finally, the coefficients $a_{k,j}$ of the smoother are optimized such that the mean square error of $\hat{h}_s(\mathbf{p}(t_s(k)),\mathcal{S}_k(d))$ is minimized. Once the MR has estimated the wireless channels at all sampling points, it selects $\mathbf{q}_{opt}$ as the optimum sampling point (i.e., with the highest estimated channel gain). This concludes the development of the CMDA.

\section{Simulations}
\label{sec:Simulations}
In this section, we present simulation results to gain more insight into the CMDA.  For the simulations setup, we select $\lambda=14.02$ cm corresponding to an RF of 2.14 GHz which is a common frequency for mobile wireless communications. For comparison purposes, we consider the same MR model with the same parameters as used in \cite{r9}.

First, we would like to highlight the benefits of executing the CMDA with the MCP derived through the method developed in section \ref{sec:CSSMDA:path}. To do this, we select a desired distance\footnote{The number of sampling points is $M=\left\lceil \frac{L_p'}{\Delta}\right\rceil+1$ where $\Delta$ is the desired distance.} between sampling points of $\Delta=0.05\lambda$. It was shown in \cite{r8} that this is sufficient to obtain the maximum channel power from each path under noiseless conditions. For the design of the paths we used sets of $N=25$ point paths.

Fig. \ref{Figure2} shows the performance\footnote{Under noiseless conditions to isolate the effect of the selected path in the performance of the CMDA.} of the CMDA when the MCP (developed in this paper) is used compared to the CLDA with a linear path (LP) and a circular path (CP). Clearly, for $L_p'\leq 0.6\lambda$, the MCP takes the form of the LP. For higher values of $L_p'$, the CMDA using the MCP outperforms the CMDA using LP both in terms of channel power obtained and in terms of the mechanical energy consumed. Although the performance improvement by using MCP instead of LP is small in terms of channel power gain, we note that in terms of mechanical energy the advantage of using the MCP is significant. This is because, due to its geometrical properties, the LP is the path that maximizes the average distance between its end point and the rest of the sampling points. By comparison, the MCP requires the MR to shorten distances during the positioning phase due to its curvatures, see Fig. \ref{Figure1}.
\begin{figure}[!t]                    
\centerline{{\label{fig:topology2}\hspace{0.005cm}{{\includegraphics[clip, trim ={8mm 2mm 22mm 0mm},width=46mm ,height=48mm]{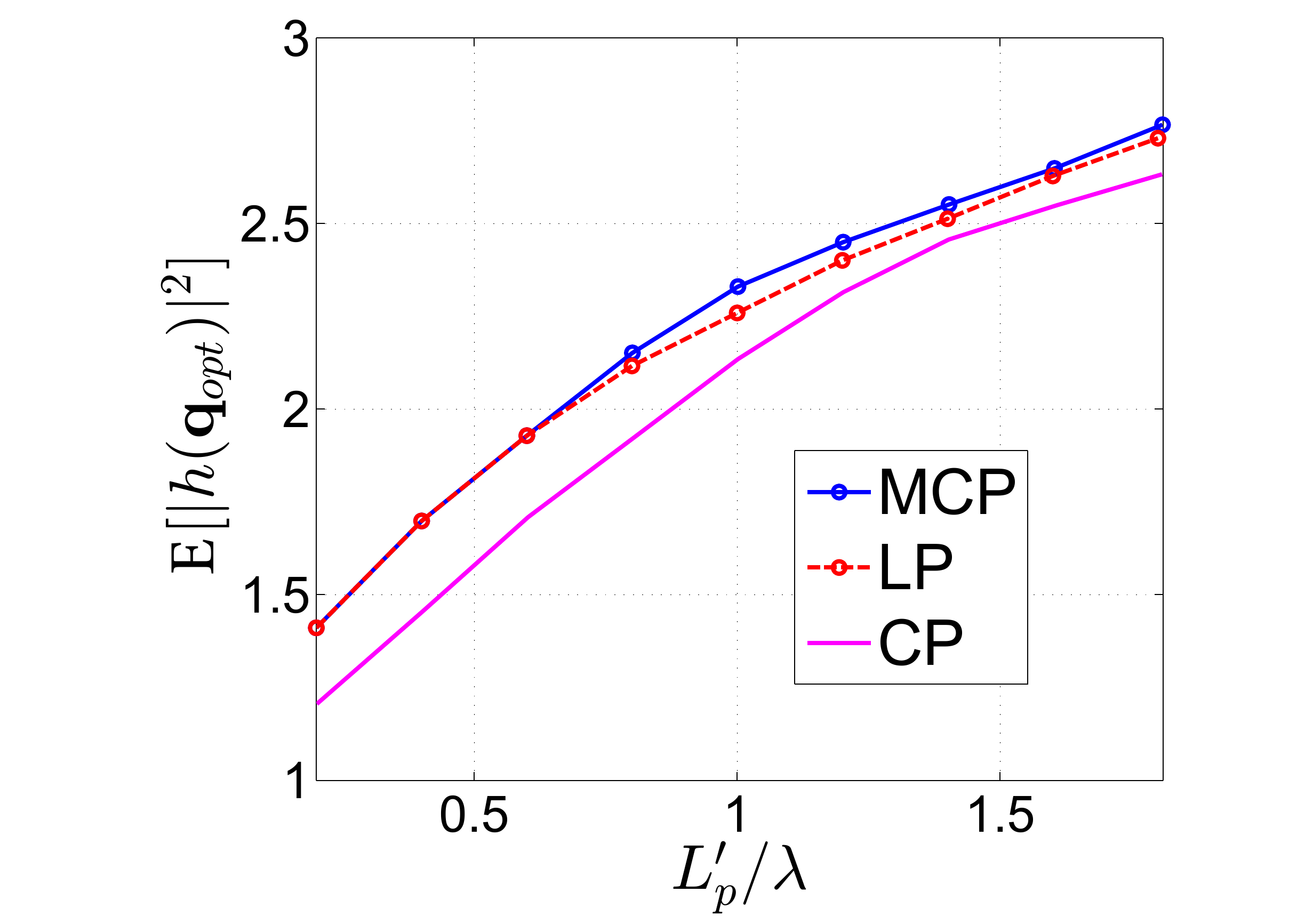}}}\hspace{-0.01cm}{{\includegraphics[clip, trim ={14mm 0mm 0mm 0mm},width=47mm ,height=48mm]{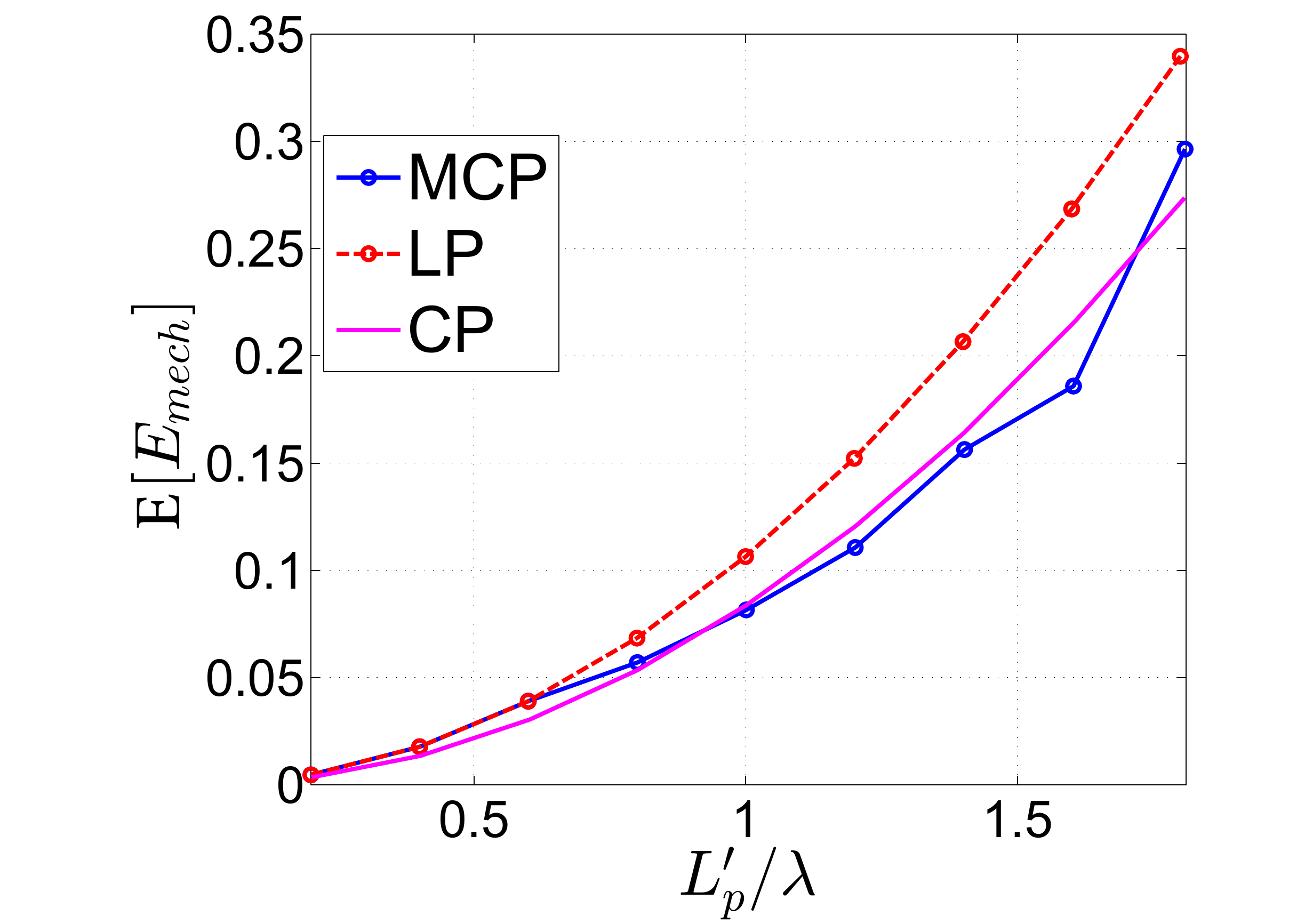}}}} }
\caption{(left) The $\mathbb{E}[|h(\mathbf{q}_{opt})|^2]$ quantity for different continuous paths and lengths under noiseless conditions for $\Delta=0.05\lambda$; (right) Mechanical energy consumption for different continuous paths and lengths under noiseless conditions for $\Delta=0.05\lambda$.}
 \label{Figure2}
\end{figure}
\begin{figure}[!htp]           
\centerline{{\includegraphics[clip, trim ={1mm 72mm 10mm 79mm},width=87mm ,height=77mm]{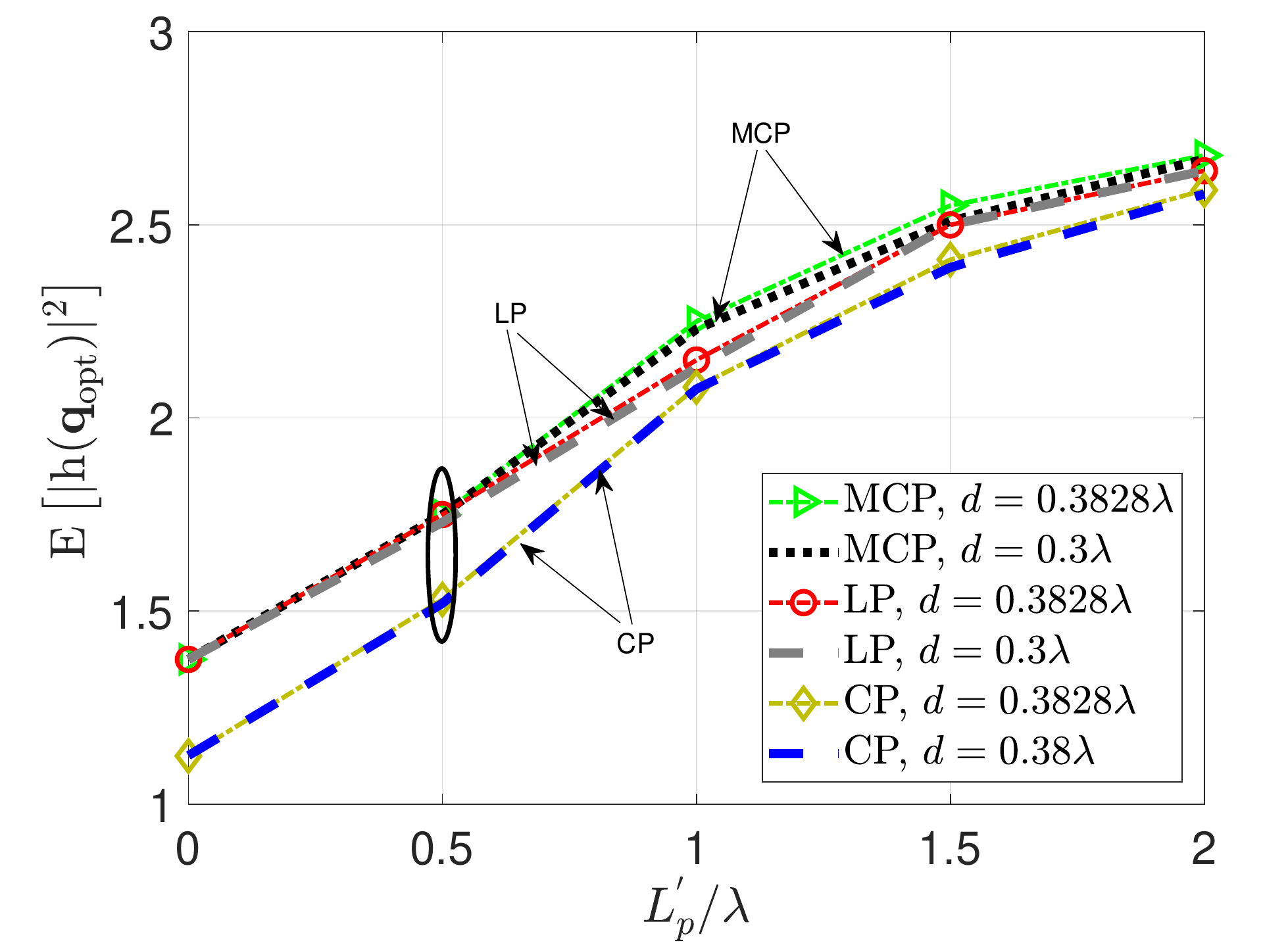}}}
    \caption[Global probability of detection for different P_t values]{\label{Figure4}
   $\mathbb{E}[|h(\mathbf{q}_{opt})|^2]$ for different continuous paths, lengths and values of $d$ with $SNR=10$dB and $\Delta=0.05\lambda$.}
\end{figure}
Nevertheless, when comparing the MCP with the CP, we observe (refer to Fig. 2) that both the CMDA using the MCP and the CMDA using the CP have, for practical purposes, almost the same performance in terms of mechanical energy. However, in terms of channel power obtained, the CMDA using the MCP has a significantly better performance compared to the CMDA using CP.

Clearly, these results show that the MCP has both the LP advantages in terms of channel power and the CP advantages in terms of mechanical energy. Hence, this enables the MDAs using the MCP to be more efficient compared to those that use either the LP or  CP.

Fig. \ref{Figure4} shows the channel power obtained by the CMDA for $SNR=10$ dB versus different values of the $d$ parameter in (\ref{eq:1.2.1}). Note that parameter $d$ determines the number of samples required for the channel estimation. First, we select $d=0.3828\lambda$. To estimate the channel at $\mathbf{p}(t_s(k))$, the MR use the samples at the sampling points that are closer to $\mathbf{p}(t_s(k))$ than the minimum distance at which the channel correlation is zero (i.e., $0.3828\lambda$). However, in the second case we select $d=0.3\lambda$ resulting in the MR using only measurements at sampling points whose channel has a correlation higher than $\approx 0.3$ with the channel to be estimated. This leads to a reduced number of samples used for each estimation and hence the computational load for the MR. However, the effect of the samples reduction on the channel power obtained by the execution of the CMDA is negligible as observed in Fig. \ref{Figure4}.

\begin{table}[ht]
\caption{MDMTA simulation results.} 
\begin{center}
\scalebox {0.96}[1]{
\begin{tabular}{|c|c||c|c|c|}
  \hline
Number of stopping points  & 3 & 4 & 5  \tabularnewline
\hline
$\mathbb{E}[E_{mech}]$   & 0.0549 & 0.1458 & 0.2753 \tabularnewline
\hline
$\mathbb{E}[|h(\mathbf{q}_{opt}^2)|]$  & 1.8301 & 2.0675 & 2.2501  \tabularnewline
\hline
\end{tabular}}
\end{center}
\label{tab:1} 
\end{table}

Finally, we compare the performance of the CMDA with the MDA parametrized on the number of stopping points. In Table \ref{tab:1}, we show the performance of the MDA based on the stopping points presented in \cite{r9} and using the optimum geometries obtained in that article. The algorithm parameters are selected in such a way that the channel power is maximized. Similarly, we  simulate the MDA with the wavelength, the MR and the execution time considered in here such that a fair comparison with the CMDA can be made. From Fig. \ref{Figure2} and Table \ref{tab:1}, we observe that for a given amount of mechanical energy used by the MR, the CMDA can provide higher channel power compared to  the MDMTA. However, the CMDA requires less MR mechanical energy in order to yield the same channel power. 

\section{Conclusions}
\label{sec:Conclusions}
In this paper, we have considered the problem of the shape optimisation of a continuous path for a given MDA and for a specified path length. The optimum path (MCP) is shown to maximize the channel power achievable via the CMDA given a fixed path length. This has been corroborated by simulations. When compared to the linear and the circular path, the MCP has shown to exhibit the advantages of both paths without their corresponding drawbacks. We have also demonstrated that for a small desired distance, the MCP takes the form of a linear path. Finally, we have shown that the CMDA with the MCP can be more efficient than MDAs using stopping points. Future work will explore and develop adaptive continuous paths enabling higher channel gains output.


\begin{thebibliography}{80}
\small
\bibitem{r10} M. Lindhe and K. H. Johansson, ``Exploiting multipath fading with a mobile robot'',The International Journal of Robotics Research October 2013 32 (12), pp. 1363-1380.
\bibitem{r11} Vieira MAM, Taylor ME, Tandon P, et al. ``Mitigating multipath fading in a mobile mesh network'', Ad Hoc Networks vol. 11, no. 4, June 2013, pp. 1510--1521.

\bibitem{r12}  A. Ghaffarkhah and Y. Mostofi, ``Path planning for networked robotic surveillance'', IEEE Transactions on Signal Procesing, vol. 60, no. 7, July 2012.
\bibitem{r9} Daniel Bonilla Licea, Mounir Ghogho, Des McLernon and Syed Ali Raza Zaidi, ``Mobility Diversity-Assisted Wireless Communication for Mobile Robots'', IEEE Transactions on Robotics, vol. 32, no. 1, 2016, pp. 214-229.
\bibitem{r0} D. Puccinelli and M. Haenggi, ``Spatial Diversity Benefits by Means of Induced Fading'', 3rd IEEE International Conference on Sensor and Ad Hoc Communications and Networks (SECON '06), Reston, VA, USA, September 2006, pp. 128-137.
\bibitem{r13} M. Lindhe, K. H. Johansson and A. Bicchi, ``An experimental study of exploiting multipath fading for robot communications'', Proc. Robotics Science and Systems, Atlanta, GA (June 2007).
\bibitem{r6} J. M. Smith, M. P. Olivieri, A. Lackpour and N. Hinnerschitz, ``RF-mobility gain: concept, measurement campaign, and explotation'', IEEE Wireless Communications, vol. 16, no. 1, February 2009.

\bibitem{r8} Daniel Bonilla Licea, Syed Ali Raza Zaidi, Des McLernon and Mounir Ghogho, ``Improving Radio Energy Harvesting in Robots using Mobility Diversity'', IEEE Transactions on Signal Processing, vol. 64, no. 8, 2016, pp. 2065-2077.
\bibitem{M05} M. K. Simon and M. Alouini, {\it Digital Communication Over Fading Channels}, Wiley. IEEE Press, 2005.

\bibitem{J94} W.C. Jakes, {\it Microwave Mobile Communications}, Wiley. IEEE Press, 2011.
\bibitem{b6} S. Russell and P. Norving, {\it Artificial Intellingence: A Modern Approach}, Prentice Hall, 2003.

\bibitem{b5} B. Siciliano, L. Sciavicco, L. Villani and G. Oriolo, {\it Robotics Modelling, Planning and Control}, Springer, 2010.




\end{thebibliography}
\end{document}